\documentclass[letterpaper, 10 pt, conference]{ieeeconf}
\IEEEoverridecommandlockouts
\usepackage{graphicx}
\usepackage{amsmath}
\usepackage{amsfonts}
\usepackage{cite}
\usepackage{lipsum}

\usepackage{amsthm}

\usepackage{algorithm,algorithmic}

\newcommand{\half}{}

\usepackage{color}

\renewcommand{\a}{\mathbf{a}}
\newcommand{\ah}{\hat{\a}}
\newcommand{\at}{\tilde{\a}}

\newcommand{\e}{\mathbf{e}}

\renewcommand{\u}{\mathbf{u}}
\newcommand{\s}{\mathbf{s}}

\renewcommand{\v}{\mathbf{v}}

\newcommand{\q}{\mathbf{q}}
\newcommand{\qd}{\dot{\q}}
\newcommand{\qdd}{\ddot{\q}}

\newcommand{\qt}{\tilde{\q}}
\newcommand{\qtd}{\dot{\qt}}

\renewcommand{\P}{\mathbf{P}}

\newcommand{\Llambda}{\boldsymbol{\Lambda}}

\newcommand{\btau}{\boldsymbol{\tau}}

\newcommand{\g}{\mathbf{g}}

\newcommand{\gh}{\hat{\g}}

\newcommand{\T}{^{T}}

\newcommand{\A}{\mathbf{A}}
\renewcommand{\L}{\mathbf{L}}
\newcommand{\K}{\mathbf{K}}
\newcommand{\G}{\mathbf{G}}

\newcommand{\Y}{\mathbf{Y}}

\newcommand{\I}{\mathbf{I}}

\newcommand{\Q}{\mathbf{Q}}
\newcommand{\W}{\mathbf{W}}
\newcommand{\M}{\mathbf{M}}

\renewcommand{\H}{\mathbf{H}}

\newcommand{\Hh}{\hat{\H}}

\newcommand{\C}{\mathbf{C}}

\newcommand{\Ch}{\hat{\C}}

\renewcommand{\d}{{\rm d}}

\renewcommand{\half}{\frac{1}{2}}

\newcommand{\delt}[2][]{\ifthenelse{\equal{#1}{}} { \delta \mathbf{#2}}{\delta \mathbf{#2}_{#1}}}

\newtheorem{theorem}{Theorem}
\newtheorem{remark}{Remark}

\newtheorem{definition}{Definition}

\renewcommand{\T}{^T}
\begin{document}

\title{Cooperative Adaptive Control for Cloud-Based Robotics\\[-1.5ex]}

\author{Patrick M. Wensing and Jean-Jacques Slotine\\[-2.5ex] %<-this % stops a space
\thanks{P. M. Wensing is an Assistant Professor in the Dept. Aerospace and Mech. Engineering, %<-
University of Notre Dame: {\tt\small pwensing@nd.edu}}%<-      
\thanks{\noindent J.-J. Slotine is a Professor in the Dept. of Mech. Engineering, %<-
        Massachusetts Institute of Technology: %<-
        {\tt\small jjs@mit.edu}}%<-
}

\maketitle

\begin{abstract}
This paper studies collaboration through the cloud in the context of cooperative adaptive control for robot manipulators. We first consider the case of multiple robots manipulating a common object through synchronous centralized update laws to identify unknown inertial parameters. Through this development, we introduce a notion of Collective Sufficient Richness, wherein parameter convergence can be enabled through teamwork in the group. The introduction of this property and the analysis of stable adaptive controllers that benefit from it constitute the main new contributions of this work. Building on this original example, we then consider decentralized update laws, time-varying network topologies, and the influence of communication delays on this process. Perhaps surprisingly, these nonidealized networked conditions inherit the same benefits of convergence being determined through collective effects for the group. Simple simulations of a planar manipulator identifying an unknown load are provided to illustrate the central idea and benefits of Collective Sufficient Richness.

% Throughout, the passivity interpretation of classical adaptive control update laws are employed to conceptualize the developments.
%Then, the mathematics of these continuous update laws are extended for asynchronous updates, providing an effective strategy for use with band-limited communication. 
\end{abstract}

\section{Introduction}
With the dawn of Internet 4.0, next-generation robots present the opportunity to benefit from dynamic global data sets through networking in the ``Cloud'' \cite{Kehoe15}. Sensing and computation may be shared through non-collocated agents distributed across wide spatial and geographic extents. Learning algorithms in applications ranging from factory automation to household cleaning may commonly benefit through shared adaptation to the challenges of the real world.  However, the dynamic nature of data/experience shared through the cloud may present as many difficulties as it does opportunities. Coupling between the complex networked updates to these datasets introduces interesting questions regarding stability of the updates and their corresponding effects on the dynamics of individual agents. In large, these considerations have been ignored in the literature, due in part to difficulties in stability analysis that arise from the nonlinear dynamics of robots as well as the time delays that are an inevitable feature of interface with cloud-based data.
	
This paper makes steps toward addressing this challenge through considering a prototype problem of cooperative adaptive control through the cloud. Adaptive control schemes seek to learn unknown parameters that affect closed-loop control outcomes while simultaneously ensuring stable trajectory tracking. This problem has a rich history \cite{Slotine87,Slotine89,Ghorbel89,Sadegh90,Niemeyer91,Yuan95,Chung09} to cite just a few (see also \cite{SlotineLi91}). This classical work has seen renewed interest (e.g., \cite{Pucci15,Bicchi17}), due in part to the increasing availability of torque-controlled robotics platforms \cite{Semini11,Englsberger14b,Hutter14b,Wensing17b}. Here we treat the case of a heterogeneous team of manipulators working with a common object that has unknown parameters, as shown in Fig.~\ref{fig:Cloud}. Such a case might occur with, for instance, multiple robots learning in tandem \cite{Levine16}. Although learning and adaptation are traditionally identified as separate in spirit, connections and commonalities (particularly for model-based sensorimotor learning) offer interesting future prospects. 
	
	The stability of networked adaptive robot control has been treated previously in a variety of other contexts. Chung and Slotine \cite{Chung09} studied synchronization of networked robots, wherein individual agents could adapt their own parameters. Schwager et al. \cite{Schwager09} studied active collective sensing with kinematic robot models and shared adaptation to learn environmental features. The current work instead considers dynamic robot models, with shared adaptation under time delays. This later consideration is of particular interest within the context of cloud-based knowledge sharing.  
	
% {\color{red}
% \noindent Vaccuum cleaner robot. Fits in introdution. Single sentence motivation, linking to deep learning. Household cleaning.
% }

	In addressing convergence to a common representation of knowledge across agents, this current work shares a great deal in common with the literature on consensus phenomena \cite{Smale07,Wang06b}. Here, we show the benefits of enforcing consensus in cooperative adaptive control. Namely,  conditions on convergence of shared knowledge to an optimum value, more technically described as conditions on sufficient richness of the learning signals, inherit relaxed requirements through teaming. The introduction of a concept of Collective Sufficient Richness is introduced, with its supporting analysis for cooperative adaptive control constituting the main contribution of this work.

%	 	adaptation for kinematic control \cite{Schwager09}

\begin{figure}
\center
\includegraphics[width = .78\columnwidth]{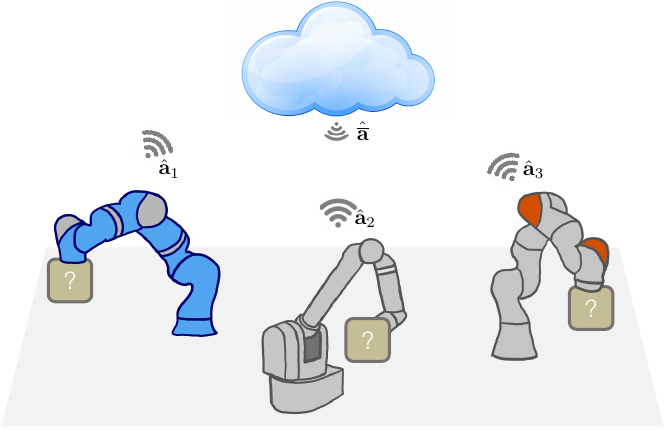}
\caption{This paper considers the classical problem of adaptive control when approached collectively through cloud-based collaboration. A team of manipulators seeks to identify a common unknown load.}
\vspace{-5px}
\label{fig:Cloud}

\end{figure}
	
	The paper is structured as follows. Section \ref{sec:prelim} describes preliminary background on manipulator adaptive control.  Section \ref{sec:synchronous} discusses the case of cooperative adaptive control under idealized network conditions, illustrating the role of collective sufficient richness in parameter convergence. Sections \ref{sec:networked} and \ref{sec:delay} then extend this analysis to handle decentralized and time-delayed communications, which may be more realistic in practice. Section \ref{sec:Composite} discusses how modeling error may be fused with tracking errors to further improve the performance of the collaborative adaptive control laws. Section \ref{sec:sim} illustrates the central benefits of the theory through a simple example in simulation.

\section{Preliminaries}
\label{sec:prelim}
In this section, we briefly introduce classical results in adaptive control for a single manipulator. We consider a traditional setting, wherein inertial parameters are identified while tracking a desired trajectory. We begin by recalling the framework of Direct Adaptive Control from \cite{Slotine87,Niemeyer91}, which uses tracking error to estimate parameters. %We then recall extensions to composite adpative control from \cite{Slotine89} that supplement adaptive update laws with terms that encode modeling errors, improving adaptation rates. 

\subsection{Direct Adaptive Control: Setup}

Consider the standard manipulator equations for a rigid-body system with $n_b$ bodies and $n_j$ degrees of freedom
\begin{equation}
\H(\q)\,  \qdd + \C(\q,\qd)\, \qd + \g(\q) = \btau
\label{eq:sysDyn}
\end{equation}
where $\q \in \mathbb{R}^{n_j}$ the vector of joint angles, $\H\in \mathbb{R}^{n_j \times n_j}$ the mass matrix, $\C\qd \in \mathbb{R}^{n_j}$ the Coriolis and centripetal terms, $\g \in \mathbb{R}^{n_j}$ the generalized gravitation force, and $\btau \in \mathbb{R}^{n_j}$ the vector of actuator torques. It is well known that one can define a vector of inertial parameters $\a\in \mathbb{R}^{10 n_b}$ that collects unknown masses, first mass moments, and moments of inertia, and that the inverse dynamics defined by \eqref{eq:sysDyn} are linear in these terms \cite{Slotine87}.
 
We consider the case of a desired trajectory $\q_d(t)\in \mathbb{R}^{n_j}$ with $\qd_d$, $\qdd_d$ bounded and  $\qdd_d$ uniformly continuous. The tracking error is defined as
\[
\qt = \q - \q_d
\] 
As a matter of notation, we denote reference velocities and accelerations as
\begin{align}
\qd_r &= \qd_d - \Llambda \qt & \qdd_r &= \qdd_d - \Llambda \qtd \nonumber 
\end{align}
where $-\Llambda$ is a constant Hurwitz matrix. If we further define
\[
\s = \qtd_r = \qd - \qd_r = \qtd + \boldsymbol{\Lambda} \qt 
\]
then regulation of $\s$ to zero renders the sliding surface defined by
\[
\qtd + \boldsymbol{\Lambda} \qt = {\mathbf{0}}
\]
to be invariant. By design, the restriction dynamics on this surface ensure $\qt$ is regulated to zero. That is, if $\s\rightarrow \mathbf{0}$ then $\qt \rightarrow \mathbf{0}$ due to the definition of $\s$. Thus, with this strategy, tracking control of $[\q_d\T,\qd_d\T]\T\in \mathbb{R}^{2n_j}$ has been reduced to the challenge of regulating $\s \in \mathbb{R}^{n_j}$ to zero. %This approach is the same principle as is employed in capture point control laws, wherein capture point tracking error plays the role of $\s$, while 

Original work in \cite{Slotine87} introduced the adaptive control law
\begin{align}
\btau &= \Hh(\q) \qdd_r + \Ch(\q,\qd) \qd_r + \gh(\q) - \K_D \s 
\label{eq:controlLaw}
\end{align}
where $\Hh$, $\Ch$, $\gh$ are estimates of the structural components of the equations of motion based on estimated parameters $\ah$. Letting the parameter error vector $\at = \ah - \a$, this choice of control law provides the following closed-loop dynamics
\begin{align}
\H \,\dot{\s} + (\K_D + \C) \s = \Y(\qdd_r,\qd_r, \qd,\q) \at
\label{eq:sdyn}
\end{align}
where $\Y(\qdd_r,\qd_r, \qd,\q)$ is the regressor \cite{Slotine87} such that
\begin{align}
\Y(\qdd_r,\qd_r, \qd,\q) \ah = \Hh(\q) \qdd_r + \Ch(\q,\qd) \qd_r  + \gh(\q) \nonumber
\end{align}
for all $\ah \in \mathbb{R}^{10 n_b}$. %If accurate values for some of these $10 n_b$ dynamic parameters are well known, \eqref{eq:sdyn} still holds when interpreting $\at$ to represent the sub-vector of errors for only the unknown parameters. 

\subsection{Direct Adaptive Control: Convergence Analysis}
Considering a Lyapunov function
\[
V = \half \s\T \H \s +  \half \at\T \P^{-1} \at
\]
with $\P=\P\T$ a fixed positive definite matrix shows that
\begin{equation}
\dot{V} = -\s\T \K_D \s + \s\T \Y \at + \at\T \P^{-1} \dot{\at} 
\label{eq:dV}
\end{equation}
Yet, since $\dot{\at}  = \dot{\ah}$, this suggests the following update law
\begin{align}
\dot{\ah} &= - \P\, \Y\T \s
\label{eq:updateLaw1}
\end{align}
where $\P$ is interpreted as the adaptation gain. Under the update law \eqref{eq:updateLaw1}, \eqref{eq:dV} simplifies to
\[
\dot{V} = -\s\T \K_D \s 
\]
Under the previous smoothness and boundedness assumptions for $\q_d(t)$, it can be shown that $\s \in \mathcal{L}_2\cap \mathcal{L}_\infty$, $\at \in \mathcal{L}_\infty$, and $\s\rightarrow \mathbf{0}$ as $t \rightarrow \infty$. Further, classical arguments \cite{Slotine87} show that if the trajectories are sufficiently rich, then parameter errors $\at \rightarrow \mathbf{0}$ as $t \rightarrow \infty$ as well. This sufficient richness condition is phrased mathematically by requiring so-called persistency of excitation as defined below.

\begin{definition}[Persistently Exciting]
A time varying matrix-valued quantity $\W(t)$ is said to be persistently exciting if there exists a $T>0$ and excitation level $\alpha>0$  such that 
\[
\frac{1}{T}\int_t^{t + T} \W\T(\tau) \W(\tau) \d \tau \ge \alpha \I \quad \forall t
\]
\end{definition}
%In saying that a matrix is persistently exciting, we mean that it is $\alpha$-persistently exciting for some $\alpha>0$.

\begin{remark}
More exactly, the definition of persistency of excitation above has subtle differences with the associated notion of sufficient richness \cite{SlotineLi87c}. However, under mild assumptions on the reference trajectories (uniform continuity of $\qdd_d$, and boundedness of $\q_d$, $\qd_d$ and $\qdd_d$), the two are equivalent \cite{SlotineLi87c}. Notions of sufficient richness and persistency of excitation are thus used interchangeably here. 
\end{remark}

\begin{theorem}[Parameter Convergence of Single-Robot Direct Adaptive Control \cite{SlotineLi87c}] 
Consider the system \eqref{eq:sysDyn} with the control law \eqref{eq:controlLaw} and parameter update law \eqref{eq:updateLaw1}. Then, if the regressor $\Y$ is persistently exciting, the parameter errors $\at$ satisfy $\at\rightarrow \mathbf{0}$ as $t\rightarrow \infty$.
\end{theorem}

\section{Design and Analysis of Synchronous Cloud-Based Update Laws}

\label{sec:synchronous}

\subsection{Setup}

 This section considers the previous update laws in the case of cloud-based adaptation. We consider the case of multiple heterogeneous robots manipulating separate copies of a common load object. To begin, the manipulators together adapt a shared estimate $\ah \in \mathbb{R}^{10}$ of the inertial parameters for the load. Although the load is the same between each robot, each manipulator can be different and may track different desired trajectories $\q_{i,d}(t)$. Inertial parameters for the manipulators themselves are initially assumed known.

Using the same control laws \eqref{eq:controlLaw} from the previous section, we consider the dynamics of $n$ systems
\begin{equation}
\H_i\, \dot{\s}_i + (\C_i + \K_{D_i}) \s_i = \Y_i\, \at, ~~i\in\{1,\ldots,n\}
\label{eq:systems}
\end{equation} 
 Note that while each system has its own tracking errors $\s_i$, there is a common  error vector $\at = \ah - \a \in \mathbb{R}^{10}$ for the load parameters,  since robots are initially assumed to be working with shared parameter information.

Toward analysis of this collective system, define
\begin{align}
\H &= {\rm diag}(\H_1, \ldots, \H_n) \nonumber &
~~~\K_D &= {\rm diag}(\K_{D_1}, \ldots, \K_{D_n}) \nonumber\\
\C &= {\rm diag}(\C_1, \ldots, \C_n)  \nonumber & ~~~\Y_{\!R} &= [\Y_1\T, \cdots, \Y_n\T]\T \\
\s &= [\s_1\T, \cdots, \s_n\T]\T  \nonumber
 \nonumber
\end{align}
where ${\rm diag}(\cdot)$ denotes a block diagonal construction, and the subscript on $\Y_{\!R}$ corresponds to the regressors being stacked row wise. Similar to the single robot direct case, this provides the overall dynamics compactly as
\[
\H\, \dot{\s} + (\C + \K_D)\s = \Y_{\!R}\, \at
\]
Consider a Lyapunov function with  contributions from tracking errors $\s_i$ and parameter errors $\at$
\begin{align}
V = \frac{1}{2} \at\T \P^{-1} \at + \frac{1}{2}\sum_{i=1}^n \s_i\T \H_i \s_i  \nonumber
\end{align}
Examining the terms from the tracking errors and the parameter vector:
\begin{align}
\dot{V} &= \at\T \P^{-1} \dot{\at}+ \s\T \Y_{\!R}\, \at  -\sum_{i=1}^n \s_i\T \K_{D_i} \s_i  \nonumber
\end{align}
If cloud communication is used to accomplish the update
\begin{align}
\dot{\ah} &= - \P\, \Y_{\!R}\T \s \label{eq:update2} = -\P \sum_{i=1}^n \Y_i\T \s_i 
\end{align}
which aggregates the updates of all robots, then again the Lyapunov function satisfies
\[
\dot{V} = - \sum_{i=1}^n \s_i\T \,\K_{D_i}\, \s_i = - \s\T \K_D \s
\]
Similar to the single robot case, this provides that each $\s_i \rightarrow \mathbf{0}$ as $t \rightarrow \infty$. While this may appear to simply generalize the result from the previous section, marked benefits exist to satisfy sufficient richness conditions collectively via teaming. 

\begin{definition}[Collective Persistency of Excitation]
A set of $n$ time-varying matrix-valued quantities $\{\W_1(t),\ldots,\W_n(t)\}$ is said to be collectively persistently exciting if there exists $T>0$ and an excitation level $\alpha > 0$ such that 
\[
\frac{1}{T} \sum_{i=1}^n \int_t^{t + T} \W_i\T(\tau) \W_i(\tau) \d \tau \ge \alpha \I \quad \forall t
\]
\end{definition}

The condition for $\{\Y_1(t),\ldots,\Y_n(t)\}$ to be collectively persistently exciting is a weaker condition than requiring any one $\Y_i(t)$ to be persistently exciting on its own. Through this insight, we state the main result for this section.

\begin{theorem}[Parameter Convergence of Multi-Robot Direct Adaptive Control with a Common Estimate] 
Consider the set of systems \eqref{eq:systems} with the control laws \eqref{eq:controlLaw} and centralized parameter update law \eqref{eq:update2}. Then, if the regressors $\{\Y_1,\ldots,\Y_n\}$ are collectively persistently exciting, the parameter errors $\at$ satisfy $\at\rightarrow \mathbf{0}$ as $t\rightarrow \infty$.
\end{theorem}

\begin{proof}
Suppose that the regressors $\{\Y_1,\ldots,\Y_n\}$ are collectively persistently exciting. Then there exits $\alpha>0$, $T>0$ such that 
\[
\frac{1}{T}\sum_{i=1}^n \int_t^{t + T} \Y_i\T(\tau) \Y_i(\tau) \d \tau \ge \alpha \I \quad \forall t
\]
which is equivalent to 
\[
\frac{1}{T} \int_t^{t + T} \Y_{\!R}\T(\tau) \Y_{\!R}(\tau) \d \tau \ge \alpha \I \quad \forall t
\]
Thus, $\Y_{\!R}$ is persistently exciting, and it follows that $\at\rightarrow \mathbf{0}$ as $t\rightarrow \infty$.
\end{proof}

\begin{remark}
Following comments as in \cite{SlotineLi87c}, since each $\q_i \rightarrow \q_{i,d}$, collective sufficient richness of the regressors $\Y_i( \qdd_{i,r}, \qd_{i,r}, \qd_i, \q_i)$ is equivalent to collective sufficient richness of the desired trajectories $\Y_i(\qdd_{i,d}, \qd_{i,d}, \qd_{i,d},\q_{i,d})$. 
\end{remark}

\renewcommand{\b}{\mathbf{b}}
\begin{remark}
The previous development can be generalized to the case where individual robots also adapt their own parameters that may be different from robot to robot. Let $\b_i$ represent unknown dynamic parameters specific to robot $i$. Then, the dynamics for $\s_i$ follow similar to \eqref{eq:sdyn} as
\[
\H_i \dot{\s}_i+ ( \C_i + \K_D) \s_i = \Y_i \at + \mathbf{Z}_i \tilde{\b}_i
\]
where $\mathbf{Z}_i$ is a suitably defined regressor for robot-specific parameters. Letting $\Q_i$ a positive definite adaptation gain for robot $i$, the update laws
\begin{align}
\dot{\ah} = -\P \sum_{i=1}^n \Y_i\T \s_i \quad  {\rm and} \quad 
\dot{\hat{\b}}_i=  -\Q_i \mathbf{Z}_i\T \s_i \nonumber
\end{align}
can be shown to result in $\s_i\rightarrow \mathbf{0} $ as $t\rightarrow\infty$ using the Lyapunov function
\[
V =\frac{1}{2} \at\T \P^{-1} \at + \frac{1}{2} \sum_{i=1}^n \left( \s_i\T \H_i \s_i + \tilde{\b}_i\T \Q_i^{-1} \tilde{\b}_i \right)
\]
This general result holds analogously for the remainder of the developments in the paper. However, focus is placed simply on load parameters $\a \in \mathbb{R}^{10}$ for clarity.  
\end{remark}

%In a sense, since the storage function $V_1$ includes additive contributions from each agent, the overall convergence rate represents an average amongst the agents. Note that if all of the systems follow the same desired trajectory, there is little benefit from this cloud-based approach. If there is heterogeneity in desired trajectories from agent-to-agent, however, tracking errors from one agent may effect those of another. Thus, in trajectory-error based adaptive control, it is unclear what benefits may exist through cloud collaboration. We see a much clearer benefit in the case of composite adaptation.

\section{Collective Sufficient Richness for Networked Update Laws}

\label{sec:networked}

Although the previous section showed the benefits of teaming toward parameter convergence, the assumption of a centralized update law may be impractical. Instead, any updates in practice will occur through delayed communication channels, requiring each robot to operate with their own estimates in addition to time-delayed information from the cloud. We continue to build toward this case, first assuming decentralized updates without delays. 

More specifically, suppose that instead of a shared parameter vector, all robots instead carry along their own estimates $\hat{\a}_i$ and update these estimates through coupling. For simplicity, all-to-all coupling is considered, however, the arguments that follow can be generalized to any fully-connected network topology. In the case of all-to-all coupling consider the system dynamics
\begin{align}
\H_i \dot{\s}_i &+ (\C_i + \K_{D_i})\s_i = \Y_i \at_i \label{eq:systems3}\\
\dot{\hat{\a}}_i &= -\P_i\left[ \Y_i\T \s_i - \frac{\K}{n} \sum_{j=1}^n ( \ah_j - \ah_i) \right] \label{eq:updateLaw3}
\end{align} 
where $\K$ a constant symmetric positive definite gain matrix. 
%{\color{red} $\K$ needs defined.}

Algebraically, as in so-called quorum sensing \cite{Tabareau10,Russo10}, this update law \eqref{eq:updateLaw3} does not require communication between all robots, but rather simply requires a common average $\hat{\overline{\a}} = \frac{1}{N} \sum_{i=1}^N \ah_i$ to be communicated to the group (through the server or the environment), with an equivalent update
\[
\dot{\hat{\a}}_i = -\P_i \left [\Y_i\T \s_i - \K (\hat{\overline{\a}} - \ah_i) \right]
\]
Note that this common average is different from the common parameter vector in Section \ref{sec:synchronous} since each system maintains its own estimate $\ah_i$ in this case.

Defining the following system-wide quantities
\begin{align}
\P &= {\rm diag}(\P_1, \ldots, \P_n) \nonumber &
\ah &= [\ah_1\T,\ldots,\ah_n\T]\T \nonumber  
 \\
\Y_{\!B} &= {\rm diag}(\Y_1, \ldots, \Y_n)\nonumber 
\end{align}
the overall networked dynamics can be written as 
\begin{align}
\H\, \dot{\s} &+ (\C + \K_D)\s = \Y_{\!B} \at \nonumber \\
\dot{\hat{\a}} &= -\P \left[ \Y_{\!B}\T \s + \L_\K \at \right] \nonumber
\end{align}
where 
\begingroup\makeatletter\def\f@size{9}\check@mathfonts
\begin{equation}
\L_\K = \frac{1}{n}\begin{bmatrix} n \K & -\K & \cdots & -\K \\
						-\K  &  n\K & \ddots & \vdots \\
						\vdots & \ddots & \ddots & -\K \\
						-\K & \cdots & -\K & n \K \end{bmatrix} \nonumber
\end{equation}
\begingroup\makeatletter\def\f@size{10}\check@mathfonts
is a block Laplacian matrix \cite{Slotine05}. It can be shown that $\L_\K$ is positive semi-definite, and $\at\T \L_\K\, \at = {\mathbf{0}}$ if and only if $\ah_i = \ah_j$ for all pairs $i$ and $j$ \cite{Slotine05}. These properties play a key role in the convergence analysis for this case.

%We note that the uncoupled system with output $\s$ is strictly output passive, and since the block Laplacian $\L_\K$ is positive semidefinite, the addition of coupling can only add dissipation and will not disrupt overall passivity. 
\subsection{Analysis}
More formally, the Lyapunov function
\begin{align}
V &= \sum_{i=1}^N \frac{1}{2} \left(\s_i\T \H_i \s_i + \at_i\T\P_i^{-1} \at_i \right) \label{eq:lyapfullconnect}
\end{align}
has rate of change
\begin{equation}
\dot{V} = -\s\T \K_D \s - \at\T \L_\K \at
\label{eq:LyapNetworkRate}
\end{equation}
Since it can be shown that $\ddot{V}$ is bounded, it follows from Barbalat's Lemma that $\s\rightarrow\mathbf{0}$ and $\at\T \L_\K \at \rightarrow \mathbf{0}$ as $t\rightarrow \infty$, and thus each $\ah_i(t)$ asymptotically converges to the common average $\hat{\overline\a}(t)$ as well.

\subsection{Collective Sufficient Richness}
However, we see something deeper here, in that networked adaptation is afforded the same benefits of parameter excitation being determined collectively, despite the decentralized update law. Intuitively, as $\K$ becomes large, \eqref{eq:updateLaw3} enforces that each $\ah_i \approx \ah_j$, approaching the centralized update law in the limit. 

\begin{theorem}[Parameter Convergence of Multi-Robot Direct Adaptive Control with a Decentralized Update Law] Consider the system \eqref{eq:systems3} with decentralized update laws \eqref{eq:updateLaw3}. Then, if the regressors $\{\Y_1,\ldots,\Y_n\}$ are collectively persistently exciting, the parameter errors $\at_i$ all satisfy $\at_i\rightarrow \mathbf{0}$ as $t\rightarrow \infty$.
\label{thm:decentralized}
\end{theorem}
\begin{proof}
Since $\s \rightarrow \mathbf{0}$, it can be shown \cite{SlotineLi87c} that for any fixed $T>0$
\[
\left\| \int_{t_2}^{t_2+T} \Y_B(t) \at(t_2) \d t \right\| \rightarrow 0  {\textrm{ as }} t_2\rightarrow \infty
\]
Yet, since each $\ah_i \rightarrow \hat{\overline{\a}}$, it follows that
\[
\left\| \int_{t_2}^{t_2+T} \Y_R(t) ( \hat{\overline{\a}}(t_2) - \a) \d t \right\| \rightarrow 0  {\textrm{ as }} t_2\rightarrow \infty
\]
From the proof of Theorem 2, $\Y_R$ persistently exciting is equivalent to collective persistency of excitation. From \cite{SlotineLi87c}, this implies $\|\hat{\overline{\a}}(t)- \a\| \rightarrow 0$, and thus, all $\ah_i \rightarrow \a$. 
\end{proof}

As a key benefit, recall again that the condition of collective sufficient richness is a much weaker condition than a condition that any individual $\Y_i$ is persistently exciting.  In summary, despite the decentralized law, parameter convergence is afforded the same excitation benefits through teaming as in the centralized case.

\subsection{Time-varying network topology}
%{\color{red} Time varying topology proof.}
The results of the previous subsection also hold in the case when the communication topology may be time varying. Such a case may occur in the event of short-term communication outages, or when information is shared only though local neighbors. At each instant, consider bidirectional communication between robot $i$ and a set of neighbors $\mathcal{N}_i(t)$. The following update law is proposed
\begin{align}
\dot{\hat{\a}}_i &= -\P_i \left[\Y_i\T \s_i -  \sum_{j \in \mathcal{N}_i(t)} \K_{ij} ( \ah_j - \ah_i)\right] \label{eq:vary_update}
\end{align} 
where each $\K_{ij} = \K_{ji}$ a symmetric positive definite coupling gain. In \cite{Wang17}, a method was recently proposed for tracking consensus in multiple robots under time-varying topologies, in effect extending the fixed topology results of \cite{Chung09}. Inspired by the analysis in \cite{Wang17}, we show that parameter consensus emerges under similar assumptions on the network structure. Further, collective persistency of excitation guarantees this consensus converges to the true parameters.

\begin{theorem}[Parameter Convergence of Multi-Robot Direct Adaptive Control with Time-Varying Networks] Consider the system \eqref{eq:systems3} with update law \eqref{eq:vary_update}. Suppose there exists a infinite set of network switching times $t_1 < t_2 < \ldots$ such that $t_{j+1}-t_j > t_d$ for some $t_d>0$ and $\forall j=1,2,\ldots$. Further suppose there exist an infinite number of uniformly bounded intervals $[t_{i_k}, t_{i_{k+1}})$ such that the union of the communication graphs over each interval are fully connected. Then, if the regressors $\{\Y_1,\ldots,\Y_n\}$ are collectively persistently exciting, the parameter errors all satisfy $\at_i\rightarrow \mathbf{0}$ as $t\rightarrow \infty$.
\end{theorem}

\begin{proof}
\newcommand{\dd}{\mathbf{d}}
Again, the Lyapunov function \eqref{eq:lyapfullconnect} satisfies 
\[
\dot{V} \le -\s\T \K_D \s
\]
and thus it can be shown that $\s\in \mathcal{L}_2$ and $\s\rightarrow \mathbf{0}$ as $t\rightarrow\infty$. The parameter error dynamics thus take the form
\begin{align}
\dot{\at}_i &= {\mathbf d}_i +\P_i \sum_{j \in \mathcal{N}_i(t)} \K_{ij} ( \at_j - \at_i) \nonumber
\end{align} 
where ${\mathbf d}_i = -\P_i \Y_i\T \s_i$ also satisfies ${\mathbf d}_i\in \mathcal{L}_2$ and ${\mathbf d}_i \rightarrow \mathbf{0}$ as $t\rightarrow \infty$. In the case when each $\mathbf{d}_i = {\mathbf{0}}$, it can be shown that the synchronization error vector $\boldsymbol{\xi} = [\at_2 - \at_1, \ldots, \at_n- \at_{n-1}]$  evolves via a uniformly exponentially stable linear time-varying system
\begin{equation}
\dot{\boldsymbol{\xi}} = \A(t) \boldsymbol{\xi}
\label{eq:unforced}
\end{equation} 
similar to \cite{Wang17}. Reconsidering the tracking errors $\s_i$ as providing an input to this system via $\mathbf{d}_{\boldsymbol{\xi}} = [\dd_2-\dd_1, \ldots, \dd_n - \dd_{n-1}]$
\begin{equation}
\dot{\boldsymbol{\xi}} =  {\A}(t) \boldsymbol{\xi} + \dd_{\boldsymbol{\xi}}
\label{eq:force}
\end{equation}
Uniform exponential stability of \eqref{eq:unforced} can be used to show both the $\mathcal{L}_2$ and $\mathcal{L}_\infty$ stability of \eqref{eq:force} \cite{Desoer75}. Therefore, the synchronization errors satisfy $\boldsymbol{\xi} \in \mathcal{L}_2 \cap \mathcal{L}_\infty$. Barbalat's lemma then immediately provides $\boldsymbol{\xi}\rightarrow \mathbf{0}$ as $t\rightarrow \infty$. With synchronization guaranteed in the limit, the proof argument of Theorem \ref{thm:decentralized} provides $\at_i\rightarrow \mathbf{0}$ for all estimates. 
\end{proof}

%Consider the virtual system with states $\x_i, \y_i$ as
%\begin{align}
%\H_i \dot{\x}_i + (\C_i + \K_D) \x_i  &= \Y_i (\y_i - \a) \\
%\dot{\y}_i &= -\Y_i\T \x_i + \K (\hat{\overline{\a}} - \y_i)
%\end{align}
%The system has both $\x_i = \mathbf{0}$, $\y_i = \hat{\overline{\a}}$, and $\x_i = \s_i$, $\y_i = \ah_i$ as solutions. The dynamics of its differential  displacement follow
%\begin{align}
%\H_i \delta \dot{\x}_i + (\C_i + \K_D) \delta \x_i &= \Y_i \delta \y_i \\
%\delta\dot{\y}_i &= -\Y_i\T \delta \x_i - \K \delta \y_i
%\end{align}
%This motivates consideration of the differential Lyapunov-like function
%\[
%\delta V = \sum_{i=1}^N \frac{1}{2} \delta \x_i\T \H_i \delta \x_i + \delta \y_i\T \delta \y_i 
%\]
%This function has rates of change
%\begin{align}
%\delta \dot{V} =& \sum \frac{1}{2} \delta \x_i\T \dot{\H}_i \delta \x_i - \delta \x_i\T( (\C_i+ \K_D) \delta \x_i - \Y_i \delta \y_i ) \nonumber \\
%	&~~~~~~~ - \delta \y_i\T \Y\T \delta \x_i - \delta \y_i\T \K \delta \y_i\T
%\end{align}
%since each $\dot{\H}_i - 2 \C_i$ is skew symmetric, this simplifies to
%\begin{align}
%\delta \dot{V} = \sum_{i=1}^N - \delta \x_i\T \K_D \delta \x_i - \delta \y_i\T \K \delta \y_i
%\end{align}
%Thus all differential displacements will decay exponentially, and the virtual system is contracting. This implies that each $\s_i\rightarrow 0$ and $\ah_i\rightarrow \hat{\overline{\a}}$ as $t\rightarrow \infty$.

\section{Collective Sufficient Richness With Time Delays}
\label{sec:delay}
This section treats the case of cloud-based teaming under communication delays. As an initial step, we consider the case where  each robot $i$ has bidirectional communication with a fixed set of neighbors $\mathcal{N}_i$. Communication from robot $i$ to robot $j \in \mathcal{N}_i$ occurs with fixed delay $T_{ij}$. It is assumed that the communication topology is fully connected. Extension to the case of time-varying communication delay and combinations with time-varying communication topologies remain topics for future analysis. In the current case, the updates are chosen to take the form
\begin{align}
\H_i \,\dot{\s}_i &+ (\C_i + \K_{D_i})\,\s_i = \Y_i \, \at_i \label{eq:delay_sys} \\
\dot{\hat{\a}}_i &= -\P_i \left[\Y_i\T \s_i -  \sum_{j \in \mathcal{N}_i} \K_{ij}( \ah_j(t-T_{ji}) - \ah_i)\right] \label{eq:delay_update}
\end{align} 
It would seem reasonable that such delays would disrupt parameter convergence and even perhaps may interfere with tracking. However, we will show that the proposed update laws are robust to time-delayed communications.

\subsection{Analysis}
Let $\G_{ij}=\G_{ji}$ such that $\K_{ij} = \G_{ij} \G_{ij}\T$. Similar to \cite{Wang06} we consider the updates through an alternative notation, inspired by the wave variable formulation \cite{Niemeyer91b} of transmission passivity \cite{Spong89c,Chopra08}. Let 
\begin{align}
 \v_{ij} &= \G_{ji}\T \ah_i & \u_{ij} = \v_{ij}(t-T_{ij})  \nonumber
\end{align}
where the subscript $ij$ indicates communication from $i$ to $j$, and define 
\begin{align}
\btau_{ji} = \u_{ji} - \v_{ij} \label{eq:tauji}	%\\
%\btau_{ij} = \u_{i} - \v_{j} \label{eq:tauij}
\end{align}
When $\K_{ij}=\I$, $\btau_{ji}=\ah_j(t-T_{ji}) - \ah_i$ represents the parameter difference following a time-delayed communication of $\ah_{j}$. When $\K_{ij}$ is chosen otherwise, $\btau_{ji}$ is roughly a scaled version of this quantity. 
The update laws \eqref{eq:delay_update} can be alternately expressed as
\begin{equation}
\dot{\hat{\a}}_i = -\P_i \left[ \Y_i\T \s_i -  \sum_{j\in \mathcal{N}_i} \G_{ji} \btau_{ji} \right]
\label{eq:tij}
\end{equation} 
Consider the modified Lyapunov function $V=\sum V_i$ with
\[
V_i =\frac{1}{2} \left[ \s_i\T \H_i \s_i + \at_i\T \P_i^{-1} \at_i + \sum_{j \in \mathcal{N}_i} \int_{t-T_{ji}}\T \v_{ji}\T \v_{ji} {\rm d}t \right]
\]
Intuitively, the integral term accounts for the parameter errors stored in the local transmission channels.

The rate of change of each term is then
% \begin{align}
% \dot{V} &= \sum_i -\s_i\T \K_D \s_i + \sum_{j\in \mathcal{N}_i} \at_i\T \G_{ji} \btau_{ji} \nonumber\\
% 		&\phantom{=\sum_i -\s_i\T \K_D \s_i} ~~+\sum_{j\in \mathcal{N}_i} \frac{1}{2}( \v_{ji}\T \v_{ji} - \u_{ji}\T \u_{ji}) \nonumber\\
% 		&= \sum_i -\s_i\T \K_D \s_i + \sum_{j\in \mathcal{N}_i} \v_{ij}\T \btau_{ji} + \frac{1}{2}( \v_{ji}\T \v_{ji} - \u_{ji}\T \u_{ji}) \nonumber
% \end{align}
% 
%
\begin{align}
\dot{V}_i &=  -\s_i\T \K_D \s_i + \!\sum_{j\in \mathcal{N}_i}\left[ \at_i\T \G_{ji} \btau_{ji} +\frac{1}{2}( \v_{ji}\T \v_{ji} - \u_{ji}\T \u_{ji}) \right]\nonumber\\
		&=  -\s_i\T \K_D \s_i + \sum_{j\in \mathcal{N}_i}\left[ \v_{ij}\T \btau_{ji} + \frac{1}{2}( \v_{ji}\T \v_{ji} - \u_{ji}\T \u_{ji}) \right] \nonumber
\end{align}

Noting that each $\u_{ji}\T \u_{ji}$ can be expressed via \eqref{eq:tauji} as
\begin{align}
\u_{ji}\T \u_{ji} = \v_{ij}\T \v_{ij} + 2 \v_{ij}\T \btau_{ji} + \btau_{ji}\T \btau_{ji}  \nonumber
\end{align}
It follows that
\begin{align}
\dot{V} &= \sum_i \left[-\s_i\T \K_D \s_i + \sum_{j\in \mathcal{N}_i} \frac{1}{2} ( \v_{ji}\T \v_{ji} - \v_{ij}\T \v_{ij} - \btau_{ji}\T\btau_{ji}) \right]\nonumber\\
		&= \sum_i \left[-\s_i\T \K_D \s_i - \sum_{j\in \mathcal{N}_i} \frac{1}{2} \btau_{ji}\T\btau_{ji} \right]\nonumber
\end{align}
Thus each $\s_i \rightarrow \mathbf{0}$ and $\btau_{ji} \rightarrow \mathbf{0}$ as $t \rightarrow \infty$. From \eqref{eq:tij} it can be seen that each $\dot{\ah}_i \rightarrow \mathbf{0}$ asymptotically. This does not necessarily mean that each $\ah_i$ converges in general. However, here, since
\begin{align}
\u_{ji}(t) &= \v_{ij}(t) + \btau_{ji}(t) = \v_{ji}(t-T_{ji}), \quad \rm{and} \nonumber \\
\u_{ij}(t) &= \v_{ji}(t) + \btau_{ij}(t) = \v_{ij}(t-T_{ij}) \nonumber
\end{align}
it follows that each $\v_{ij}$ converges to a trajectory with period $T_{ij}+T_{ji}$. Yet, since $\dot{\v}_{ij}=\G_{ji}\T \dot{\ah}_i$ and $\dot{\ah}\rightarrow \mathbf{0}$, the only possibility for this periodic trajectory is that each $\v_{ij}$ and $\ah_i$ converge to constant values. Further, since each $\btau_{ij}\rightarrow \mathbf{0}$, $\ah_i \rightarrow \ah_j$ for every pair of neighbors. Thus, all the estimates $\ah_i$ must converge to some fixed common constant $\overline{\a}$.

\begin{theorem}[Parameter Convergence of Multi-Robot Direct Adaptive Control with Time-Delayed Decentralized Updates] Consider the system \eqref{eq:delay_sys} with decentralized update laws \eqref{eq:delay_update}. Then, if the regressors $\{\Y_1,\ldots,\Y_n\}$ are collectively persistently exciting, the parameter errors $\at_i$ all satisfy $\at_i\rightarrow\mathbf{0}$ as $t\rightarrow \infty$.
\end{theorem}
\begin{proof}
Since each $\ah_i \rightarrow \overline{\a}$ for some fixed common constant $\overline{\a}$, convergence of each $\at_i \rightarrow \mathbf{0}$ follows through identical arguments to the proof of Theorem \ref{thm:decentralized}.
\end{proof}

\begin{remark}
The system behavior in this case may be interpreted in terms of combination properties of contracting systems \cite{Slotine98}, \cite{Slotine05}, \cite{Wang06}. Both the uncoupled dynamics for $(\s, \at)$ and the additional time-delayed Laplacian coupling can be shown to be semi-contracting in a common, block diagonal  metric  $\M = {\rm diag}( \H, \P^{-1} )$,
making their parallel combination  semi-contracting.
\end{remark}

\vspace{-5px}
\section{Extensions via Composite Adaptation}
\vspace{-0px}
\label{sec:Composite}

Beyond using the tracking errors to drive adaptation, any modeling errors linear in the unknown parameters can be used to drive parameter update laws. Composite adaptive control is a method to increase the convergence rate of  direct adaptive control by fusing information related to these modeling errors. We address a number of potential inclusions of these composite updates to improve the performance of cloud-based networked adaptive control.
 
%{\color{red} Conversely to filtered energy, I guess that, in the composite, if computation is not an issue, one could use a bank of linear or nonlinear filters to enhance p. of e.}
 \vspace{-0px}
\subsection{Composite Adaptive Control: Filtered Torques}
\vspace{-0px}
One possible addition is to filter the dynamics \cite{Slotine89}
\begin{align}
\frac{\lambda}{p + \lambda} \btau = \frac{\lambda}{p + \lambda}\left[ \H \qdd + \C \qd + \g \right]
\label{eq:filterA}
\end{align}
where $p$ the Laplace variable. Writing the right hand side as a convolution integral with the impulse response of the low-pass filter, integration by parts can be used to yield
\begin{align}
\frac{\lambda}{p + \lambda} \btau = \lambda \H \qd  - \frac{\lambda}{p + \lambda}\left[ \lambda \H \qd +\C\T\qd - \g \right]
\label{eq:filterB}
\end{align}
Since all entries of the right hand side are linear in $\a$, one can construct a matrix $\W$, dependent on the time histories of $\q$ and $\qd$, such that
\begin{align}
\frac{\lambda}{p + \lambda} \btau = \W(\q,\qd) \a \nonumber
\end{align} 
Letting the modeling error $\e= \W \ah - \lambda/(p+\lambda) \btau$, then
\begin{align}
\e = \W(\q,\qd) \at \nonumber
\end{align}
To explore the effects of including composite adaptation, consider again the fully-connected networked systems without time delay:
\begin{align}
\H_i \dot{\s}_i &+ (\C_i + \K_{D_i})\s_i = \Y_i \at_i \label{eq:systems32}\\
\dot{\hat{\a}}_i &= -\P_i\left[ \Y_i\T \s_i + \W_i\T \e_i- \frac{\K}{n} \sum_{j=1}^n ( \ah_j - \ah_i) \right] \label{eq:updateLaw32}
\end{align} 
Through using composite adaptation, the Lyapunov function \eqref{eq:lyapfullconnect} has rate of change 
%.
%\[
%\W(\q,\qd) = \frac{\lambda}{p+\lambda} \Y(\q,\qd,\qdd)
%\]
%This regressor $\W$ can be used to compare filtered measured torques to the filtered expected torques, providing a measure of modeling error for each component of the equations of motion through
%\[
%\e = \W \ah  - \frac{\lambda}{p+\lambda} \btau = \W \at
%\]
\begin{align}
\dot{V} &= - \at\T \L_\K \at - \sum_{i=1}^n\left(\s_i\T \K_D \s_i + \at_i\T \W_i\T \e_i\right) \nonumber \\
		 &= - \at\T \L_\K \at - \sum_{i=1}^n\left(\s_i\T \K_D \s_i + \e_i\T \e_i\right) \nonumber 
\end{align}
which is more negative in comparison to the direct case \eqref{eq:LyapNetworkRate}. Note that the only property used here is $\e_i = \W_i \at_i$, and thus one can pursue other error measures linear in the parameters. In a sense, the terms $\e_i\T \e_i = \at_i\T \W_i\T \W_i \at_i$ can be viewed as adding positive semidefinite terms to the diagonal of the block Laplacian $\L_\K$, which may be positive definite on average in the case of persistency of excitation for each $\W_i$.

% all these contintuous approximation anyways
% As opposed to other approximations 

\begin{remark}
Implementation of this nonlinear strategy with sampled data will necessarily imply approximations. In discrete time implementations, one may consider a discrete low-pass filter applied directly in \eqref{eq:filterA}. 
\begin{equation}
\frac{1-\gamma}{1 - \gamma z^{-1}} \btau = \frac{1-\gamma}{1 - \gamma z^{-1}}  \left[ \H \qdd + \C \qd + \g \right] \nonumber
\end{equation}
where $z^{-1}$ is the unit delay operator and $\gamma \in (0,1)$ sets the filter bandwidth. Summation by parts can be applied to approximate the impulse response of these filters and avoid measurement of $\qdd$ via 
\begin{equation}
\frac{1-\gamma}{1 - \gamma z^{-1}} \btau = \beta \H \qd - \frac{1-\gamma}{1 - \gamma z^{-1}} \left[ \beta \H \qd +\C\T\qd - \g \right]
\label{eq:lowpassdiscrete}
\end{equation}
where $\beta = (1-\gamma)\gamma^{-1} T^{-1}$ for a sample time $T$. The scaling factor $\beta$ on the generalized momentum $\H \qd$ depends on both cutoff frequency and sample time. This accounts for a fully discrete implementation of the filter, unlike a time discretization of \eqref{eq:filterB} where the scaling of the generalized momentum is only frequency dependent.
\end{remark}

\vspace{-0px}
\subsection{Composite Adaptive Control: Filtered Energy}
\vspace{-0px}
Another possible choice of composite error injection follows from considering the rate of change in energy \cite{Slotine89,Niemeyer93}
\[
\qd\T \btau = \frac{ {\rm d}}{ {\rm d} t}\left[ \frac{1}{2} \qd\T \H \qd \right] + \qd\T\g
\]
Again, filtering both sides, taking the impulse response of the filter convolution, and applying integration by parts, it can be shown that 
\[
\frac{\lambda}{p + \lambda} \left[ \qd\T \btau \right] = \frac{\lambda}{2} \qd\T \H \qd - \frac{\lambda}{p+\lambda} \left[\frac{\lambda}{2} \qd\T \H \qd - \qd\T \g \right]
\]
Similar to before, the right hand side can be formed from the time histories of $\q, \qd$ and is linear in parameters $\a$. Thus, there exists a modified filtered regressor $\W(\q,\qd)$ such that
\begin{align}
\frac{\lambda}{p + \lambda} \left[ \qd\T \btau \right] &= \W\, \a \nonumber \\
\e = \W \ah - \frac{\lambda}{p + \lambda}\left[ \qd\T \btau \right] &= \W \at \nonumber 
\end{align}
Thus all of the previous results hold with the filtered energy regressor in place of the filtered torques regressor.

\subsection{General Convergence Results for Composite Adaptation}
More broadly, the update law \eqref{eq:updateLaw32} could be employed with heterogeneous choices for the extra data fused at each node. That is, some robots might use filtered energy for additional fusion, while others use the filtered torques. The main message is that this additional fusion can only help the convergence process, and that this positive effect is shared through the network. We state this result more formally:

\begin{theorem}[Parameter Convergence of Multi-Robot Composite Adaptive Control wit Decentralized Updates] Consider the system \eqref{eq:delay_sys} with decentralized update laws \eqref{eq:delay_update}. If the  regressors $\{\Y_1,\ldots,\Y_n,\W_1, ..., \W_n\}$
 are collectively persistently exciting, then the parameter errors $\at_i$ all satisfy $\at_i\rightarrow \mathbf{0}$ as $t\rightarrow \infty$. Further, if $\{\W_1, ..., \W_n\}$ are collectively persistently exciting with excitation level $\alpha$, parameter errors converge exponentially with a rate determined by $\alpha$.
\end{theorem}

\begin{remark}
 While composite adaptive control was proposed in the late 1980s, applications were originally limited due to computation requirements. However, a million-fold increase in computing power since its original development and associated algorithmic advances \cite{WangH12}  make the approach increasingly viable. The inclusion of composite adaptation also introduces the opportunity to include time-varying adaptation gains $\P_i(t)$ that respond to the excitation level of the input, for instance, through bounded-gain least-squares forgetting or other conceptually-similar methods \cite{Slotine89}.

\end{remark}

\begin{remark}
In light of the previous remark, it may be desirable to fuse multiple sources of composite adaptation. For instance, one may consider fusing multiple instances of filtered torques/energy with different cutoff frequencies. Collections of linear time-varying filters might also be considered for additional flexibility. Such filter banks would enhance collective persistency of excitation, and thus the convergence rate of the parameters. Such techniques may be viewed as an instance of model-based boosting akin to techniques in machine learning that combine weak learners to obtain better predictions (e.g. \cite{Liu17b}).  Extensions to filters with linear ramp units might also be considered by exploiting their input-output scaling properties.
\end{remark}

\vspace{-0px}
\section{Simulation Results}
\vspace{-0px}
\label{sec:sim}
To show the benefits and validity of the new theory, this section reports on a simple set of simulations. These simulations illustrate the central benefits of Collective Sufficient Richness through the simplest essence of the concept. We consider two planar 3-DoF robots as depicted in Fig.~\ref{fig:sim} attempting to identify a common unknown load.  

\begin{figure}
\center
\includegraphics[width=.5 \columnwidth]{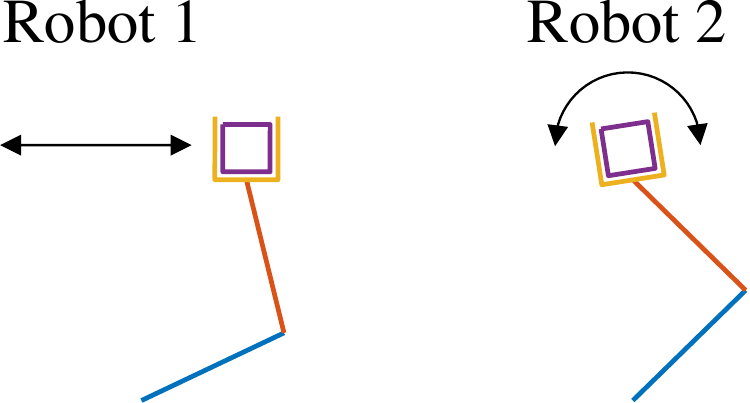}
\caption{Conceptual illustration of desired trajectories for two planar robots with a common unknown load.}
\label{fig:sim}
\end{figure}

\begin{figure}
\center
\includegraphics[width =  .95 \columnwidth]{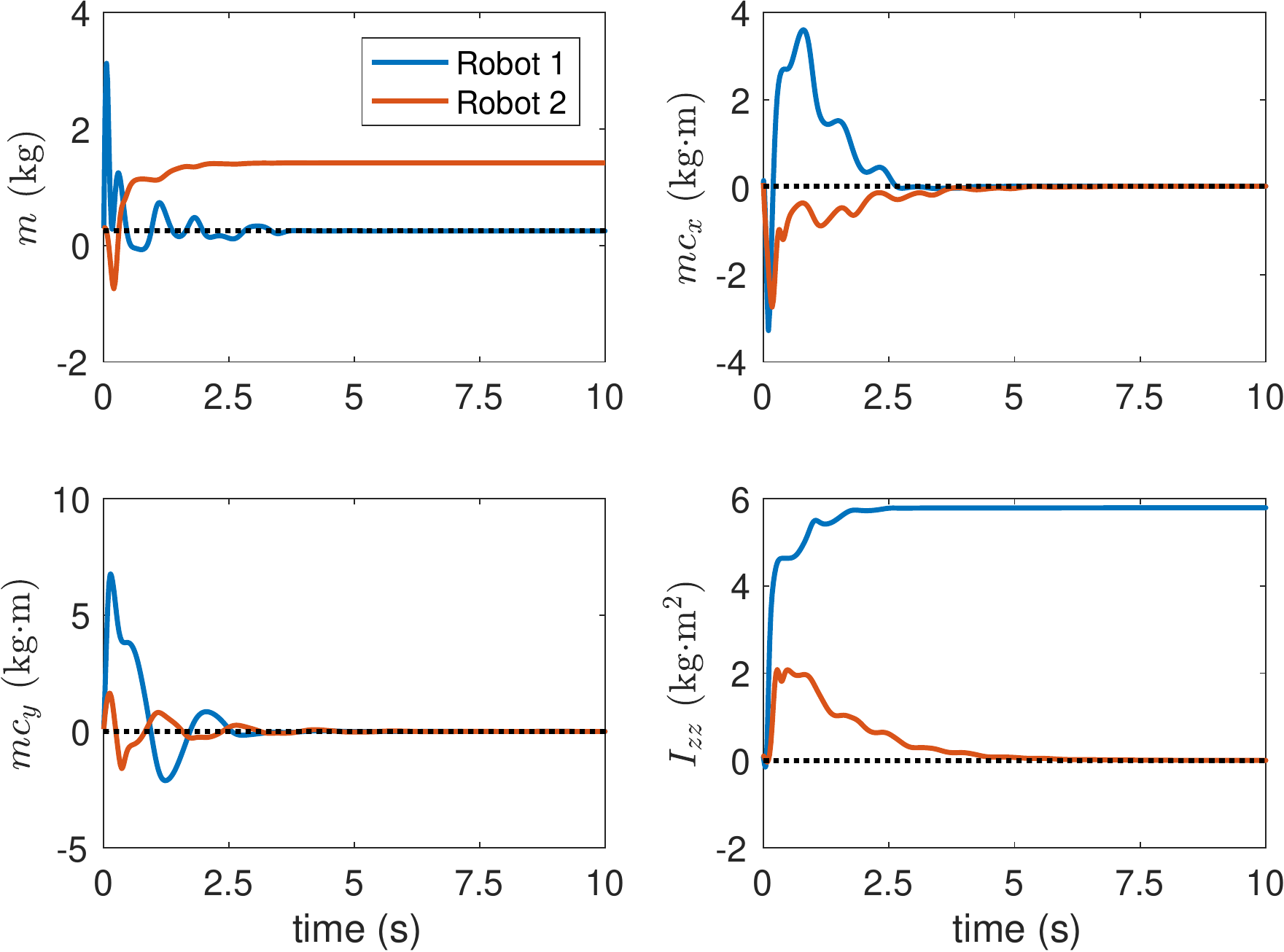}
\caption{Parameter convergence of decoupled robots.}
\label{fig:noteam}
\end{figure}

\begin{figure}[t]
\center
\includegraphics[width = .95 \columnwidth]{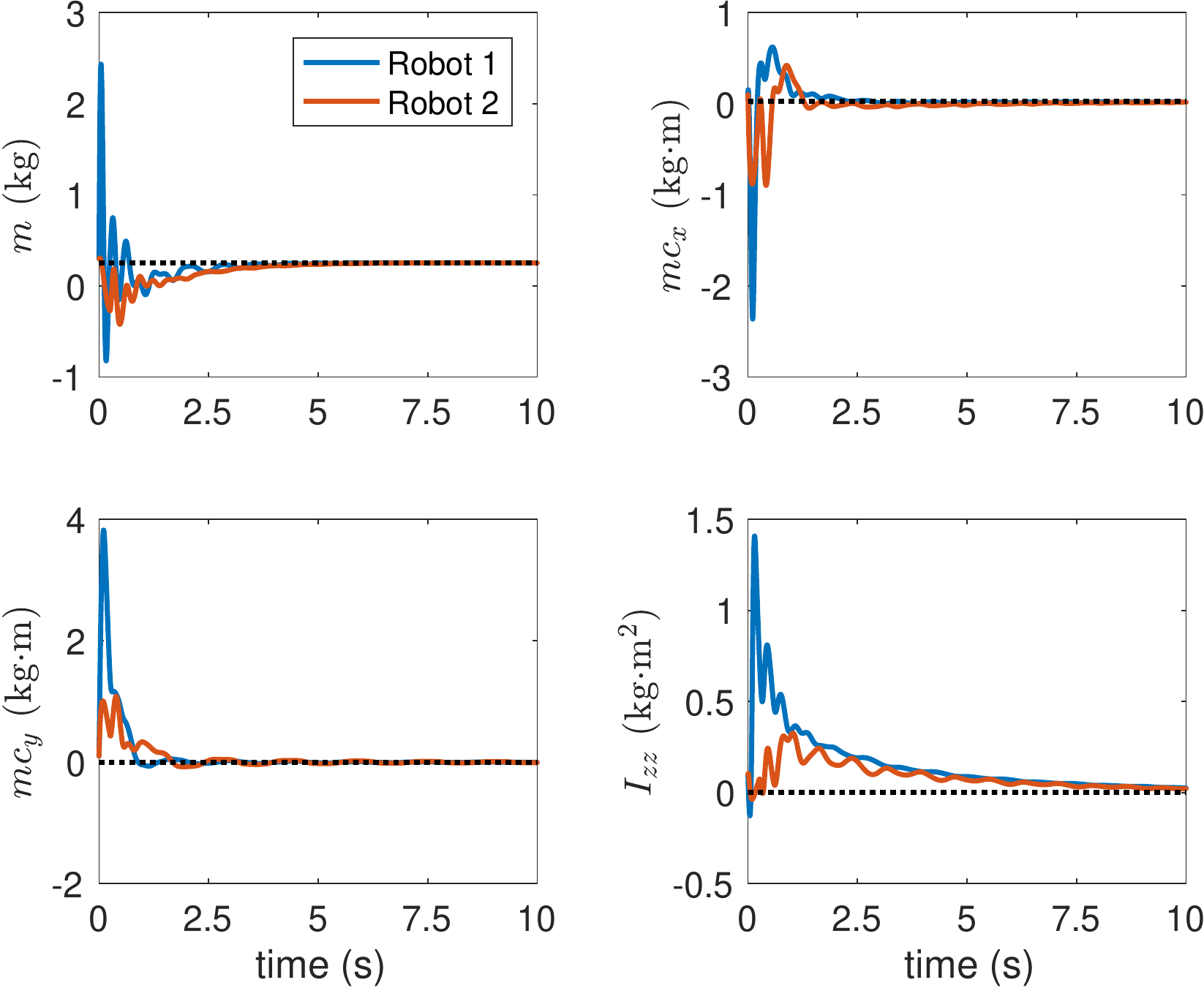}
\caption{Parameter convergence with time-delayed coupling.}
\label{fig:team}

\end{figure}

The case of no teaming when $\K=0$ is shown in Fig.~\ref{fig:noteam}. Other parameters are set as $\boldsymbol{\Lambda} = 4$ s$^{-1}$, $\K_D = 4$ N$\cdot$ m $\cdot$s/rad. Due to the trajectories followed, the first robot is unable to identify the rotational inertia of the unknown load, while the second robot cannot identify the unknown mass of the load.

Although the trajectories of neither robot are sufficiently rich, when considered together, there is enough information content in the combined signals to satisfy collective sufficient richness conditions.  Fig.~\ref{fig:team} shows this case with coupling gain $\K=5$ and a communication time delay of $0.25$s. The parameters converge to their true values despite delays thanks to benefits of collaboration in this case. Note that the mass temporarily goes negative in both examples, highlighting the potential advantage of exploiting  physical consistency constraints  \cite{Slotine86,Ioannou12,Wensing18} to be included in the adaptation laws. This inclusion represents an area for future improvements.

Figure \ref{fig:composite} illustrates the benefits of composite adaptation on the teaming process. The results show the same gain settings as Fig.~\ref{fig:team} in the absence of communication delay. Removing the delay ensures $V \rightarrow 0$ as $t\rightarrow \infty$ under collective persistency of excitation. The vertical axis is presented on a log scale, such that the slope of the curve represents the exponential rate of convergence. It is observed that the inclusion of filtered torques results in a faster convergence than filtered energy, and that fusing filtered torque and filtered energy regressors only minimally improves convergence.

\begin{figure}[t]
\center
\includegraphics[width = .85\columnwidth]{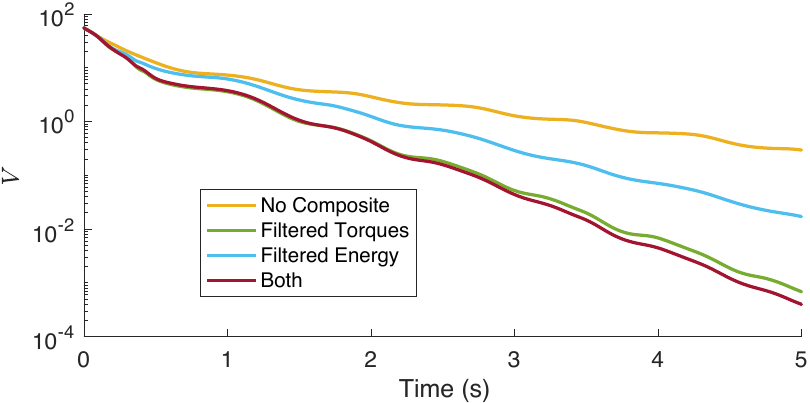}
\caption{Parameter convergence with collective composite adaptation.}
\label{fig:composite}

\end{figure}

%Figure \ref{fig:composite} shows the effects of collective composite adaptation. The figure compares the inclusion of different composite terms in the adaptive law for the case of teaming with no time delay. Note that in the no-delay case, under persistency of excit

\vspace{-0px}
\section{Conclusion}
\vspace{-0px}

This paper has treated collaboration through the cloud in the context of cooperative adaptive control for robot manipulators. We have introduced new decentralized parameter update laws and proved their convergence through the introduction of a new characterization of collective sufficient richness. With the update laws proposed, this property holds robustly with time-delayed communication or varying network topology, and can be extended to improve convergence rates through the inclusion of composite adaptation. Future work will consider extensions to underactuated systems, the inclusion of physical consistency constraints into the adaptive update laws, and the combination of time-varying delays with time-varying network topology.

\bibliographystyle{ieeetr}
% \bibliography{sysid}
\bibliography{CloudAdaptation}

\end{document}